\def\year{2020}\relax
\documentclass[letterpaper]{article} 
\usepackage{aaai20}  
\usepackage{times}  
\usepackage{helvet} 
\usepackage{courier}  
\usepackage[hyphens]{url}  
\usepackage{graphicx} 
\usepackage{graphicx}  
\usepackage{amssymb}
\usepackage{amsmath}
\usepackage{amsthm}
\usepackage{tikz}
\usetikzlibrary{positioning}
\newcommand\numberthis{\addtocounter{equation}{1}\tag{\theequation}}
\newcommand\ci{\perp\!\!\!\perp} 
\DeclareMathOperator*{\argmax}{arg\,max}
\newtheorem{theorem}{Theorem}[section]

\newcommand{\citet}[1]{\citeauthor{#1} \shortcite{#1}}
\title{Hypernym Detection Using Strict Partial Order Networks}

\author{
  Sarthak Dash, Md Faisal Mahbub Chowdhury, Alfio Gliozzo,   \\
  {\bf \Large Nandana Mihindukulasooriya \and  Nicolas Rodolfo Fauceglia} \\
 IBM Research AI, Yorktown Heights, NY, USA\\
 sdash@us.ibm.com, mchowdh@us.ibm.com, gliozzo@us.ibm.com \\
 nandana.m@ibm.com, nicolas.fauceglia@ibm.com\\
}

\begin{document}

\maketitle

\begin{abstract}
This paper introduces Strict Partial Order Networks (SPON), a novel neural network architecture designed to enforce asymmetry and transitive properties as soft constraints. We apply it to induce hypernymy relations by training with \emph{is-a} pairs. We also present an \textit{augmented} variant of SPON that can generalize type information learned for in-vocabulary terms to previously unseen ones. An extensive evaluation over eleven benchmarks across different tasks shows that SPON consistently either outperforms or attains the state of the art on all but one of these benchmarks.
\end{abstract}

\section{Introduction}

The ability to generalize the meaning of domain-specific terms is essential for many NLP applications. However, building taxonomies by hand for a new domain is time-consuming. This drives the requirement to develop automatic systems that are able to identify \textit{hypernymy} relationships (i.e.\ \emph{is-a} relations) from text.

Hypernymy relation is reflexive and transitive but not symmetric \cite{george1990introduction,hearst1992automatic}. For example, if \emph{Wittgenstein} $\prec$ \emph{philosopher} and \emph{philosopher} $\prec$ \emph{person}, where $\prec$ means \emph{is-a}, it follows that \emph{Wittgenstein} $\prec$ \emph{person} (\emph{transitivity}). In addition, it also follows that both \emph{philosopher} $\nprec$ \emph{Wittgenstein} and \emph{person} $\nprec$ \emph{philosopher} (\emph{asymmetry}). Absence of self-loops within taxonomies (e.g.\ WordNet \cite{george1990introduction}) emphasizes that reflexivity (e.g.\ \emph{person} $\prec$ \emph{person}) does not add any new information.

In \emph{order theory}, a \texttt{partial order} is a binary relation that is transitive, reflexive and anti-symmetric. A \texttt{strict partial order} is a binary relation that is transitive, irreflexive and asymmetric. \emph{Strict partial orders} correspond more directly to \emph{directed acyclic graphs} (\texttt{DAGs}). In fact, hypernymy relation hierarchy in WordNet is a \emph{DAG} \cite{suchanek2008yago}. Therefore, we hypothesize that the \emph{Hypernymy} relations within a taxonomy can be better represented via \emph{strict partial order} relations. 

In this paper we introduce \texttt{Strict Partial Order Networks (SPON)}, a neural network architecture comprising of \emph{non-negative} activations and \emph{residual} connections designed to enforce strict partial order as a soft constraint. We present an implementation of SPON designed to learn \emph{is-a} relations. The input of SPON is a list of \emph{is-a} pairs, provided either by applying Hearst-like patterns over a text corpus or via a list of manually validated pairs.

In order to identify hypernyms for out-of-vocabulary (OOV) terms, i.e. terms that are not seen by SPON during the training phase, we present an \textit{augmented} variant of SPON that can generalize type information learned for the in-vocabulary terms to previously unseen ones. The \textit{augmented} model does so by using normalized distributional similarity values as weights within a probabilistic model, the details of which are described in Section \ref{oovterms}.

The main contributions of this paper are the following:
\begin{itemize}
    \item We introduce the idea of Strict Partial Order Network (SPON), highlighting differences and similarities with previous approaches aimed at the same task.
    \item A theoretical analysis shows SPON enforces asymmetry and transitivity requirement as soft constraints .
    \item An \textit{augmented variant} of SPON to predict hypernyms for OOV is proposed.
    \item Compared to previous approaches, we demonstrate that our system achieves and/or improves the state of the art (SOTA) consistently across a large variety of hypernymy tasks and datasets (multi-lingual and domain-specific), including supervised and unsupervised settings.
\end{itemize}

The rest of the paper is structured as follows. Section \ref{relatedwork} describes related work. SPON is introduced in Section \ref{SPON}, and theoretical analysis is provided in Section \ref{theoreticalanalysis}. In Section \ref{oovterms} we show how SPON can be \textit{augmented} for OOV terms in the test dataset. Section \ref{unsupervised} and \ref{supervised} describe the evaluation setup and results. Section \ref{conclusion} concludes the paper and highlights perspectives for future work.

\section{Related Work}\label{relatedwork}
Since the pioneering work of \citet{hearst1992automatic}, lexico-syntactic pattern-based approaches (e.g., ``NP$_y$ is a NP$_x$'')  remains influential in subsequent academic and commercial applications. Some work tried to learn such patterns automatically~\cite{snow2005learning,shwartz2016improving} instead of using a predefined list of patterns. 

Among other notable work, \citet{kruszewski2015deriving} proposed to map concepts into a boolean lattice. \citet{lin2002concept} approached the problem by clustering entities.  \citet{dalvi2012websets} proposed to combine clustering with Hearst-like patterns. There also exist approaches \cite{weeds2004characterising,roller2016relations,shwartz2017hypernyms}  inspired by the Distributional Inclusion Hypothesis (DIH) \cite{geffet2005distributional}. 

\citet{fu2014learning} argued that hypernym-hyponym pairs preserve linguistic regularities such as $v(shrimp)-v(prawn)$
$\approx$ $v(fish)-v(goldfish)$, where $v(w)$ is the embedding of the word $w$. In other words, they claimed that a hyponym word can be projected to its hypernym word learning a transition matrix $\Phi$. \citet{tan2015usaar} proposed a deep neural network based approach to learn \emph{is-a} vectors that can replace $\Phi$. 

Recently, \citet{roller2018hearst} showed that exploitation of matrix factorization (MF) on a Hearst-like pattern-based system's output vastly improved their results (for different hypernymy tasks; in multiple datasets) with comparison to that of both distributional and non-MF pattern-based approaches.

Another thread of related work involves the use of graph embedding techniques for representing a hierarchical structure. Order-embeddings \cite{Vendrov:2016} encode text and images with embeddings, preserving a \emph{partial order} (i.e. $x \preceq y$, where x is a specific concept and y is a more general concept) over individual embedding dimensions using the \textit{Reversed Product Order} on $\mathbb{R}^N_+$. In contrast, our proposed neural network based model encodes a \emph{strict partial order} through a composition of \textit{non-linearities} and \textit{residual} connections. This allows our model to be as \textit{expressive} as possible, all the while maintaining strict partial order.

\citet{LiXiang:2017} extended the work of \citet{Vendrov:2016} by augmenting distributional co-occurrences with order embeddings. In addition, hyperbolic embeddings model tree structures using non-euclidean geometries, and can be viewed as a continuous generalization of the same \cite{nickel2017poincare}. Other recent works have induced hierarchies using box-lattice structures \cite{vilnis2018probabilistic} and Gaussian Word Embeddings \cite{athiwaratkun2018hierarchical}.

Regarding the recent SOTA, for unsupervised setting where manually annotated (i.e. gold standard) training data is not provided, \citet{le2019Hyperbolic} proposed a new method combining hyperbolic embeddings and Hearst-like patterns, and obtained significantly better results on several benchmark datasets. 

For supervised setting, during the SemEval-2018 hypernymy shared task \cite{camacho-collados-etal-2018-semeval}, the \texttt{CRIM} system \cite{bernier2018crim} obtained best results on English datasets (General English, Medical and Music). This system combines supervised projection learning with a Hearst-like pattern-based system's output. In the same shared task, for Italian, the best system, \texttt{300-sparsans}, was a logistic regression model based on sparse coding and a formal concept hierarchy obtained from word embeddings \cite{berend2018300}; whereas for Spanish, the best system, \texttt{NLP\_HZ} was based on the nearest neighbors algorithm \cite{qiu2018nlp_hz}.

In Sections \ref{unsupervised} and \ref{supervised} we compare our approach with all of the above mentioned recent SOTA in both unsupervised and supervised settings, respectively.

\section{Strict Partial Order Networks}
\label{SPON}

The goal of SPON is to estimate the probability for a distinct pair of elements $x,y \in \mathcal{E}$ to be related by a strict partial order $x \prec y$. A specific instance of this problem is the hypernym detection problem, where $\mathcal{E}$ is a vocabulary of terms and $\prec$ is the \emph{is-a} relation. In this section, we present a SPON implementation, while a theoretical analysis of how the proposed architecture satisfies transitive and asymmetric properties is described in the next section.

An implementation of a SPON is illustrated in Figure \ref{architecture}.
Each term $x \in \mathcal{E}$ is represented via a vector $\vec{x} \in \mathbb{R}^d$. In the first step, we perform an element-wise multiplication with a weight vector $w_1$ and then add to a bias vector $b_1$. The next step consists of a standard \textit{ReLU} layer, that applies the transformation $ReLU(v) = max(0, v)$. Let us denote these transformations by a smooth function $g$,
\begin{equation} \label{reludesc}
g(\vec{x}) = ReLU(w_1 \otimes \vec{x} + b_1)
\end{equation}
where $\otimes$ denote element-wise multiplication. 

The final step, as depicted in Figure~\ref{architecture}, consists of a residual connection, i.e. 
\begin{equation} \label{resconn}
f(\vec{x}) = \vec{x}+g(\vec{x})
\end{equation}


\begin{figure}[ht]
    \centering
    \begin{tikzpicture}[
    roundnode/.style={circle, draw=black!60, very thick, minimum size=7mm},
    squarednode/.style={rectangle, draw=black!60, very thick, minimum size=5mm},
    point/.style={circle, draw=black!5, ultra thin, minimum size=0.1mm},
    on grid
    ]
    \node[squarednode]      (inpvector)                     {Input vector $\vec{x}$};
    \node[roundnode]        (elemmul)       [above=12mm of inpvector] {\bf{$\otimes$}};
    \node[roundnode]      (elemadd)       [above=12mm of elemmul] {\bf{+}};
    \node[squarednode]      (bias)          [left=20mm of elemadd] {Bias Vector $b_1$};
    \node[squarednode]      (relu)          [above=12mm of elemadd] {ReLU Layer};
    \node[squarednode]      (wmatrix)       [left=20mm of elemmul] {Weight Vector $w_1$};
    \node[roundnode]      (elemadd2)       [above=12mm of relu] {\bf{+}};
    \node[squarednode]     (loss)          [above=12mm of elemadd2] {Loss Layer};
    \node[point] (p1)        [right=20mm of inpvector] {};
    \node[point] (p2)        [right=20mm of elemadd2] {};
    
    \draw[->] (inpvector.north) -- (elemmul.south);
    \draw[->] (elemmul.north) -- (elemadd.south); 
    \draw[->] (elemadd.north) -- (relu.south);
    \draw[->] (bias.east) -- (elemadd.west);
    \draw[->] (wmatrix.east) -- (elemmul.west);
    \draw[->] (relu.north) -- (elemadd2.south); 
    \draw[->] (elemadd2.north) -- (loss.south); 
    \draw[-]  (inpvector.east) -- (p1.west) -- (p2.west); 
    \draw[->] (p2.west) -- (elemadd2.east);
    \end{tikzpicture}
    \caption{Simple SPON architecture.}
    \label{architecture}
\end{figure}
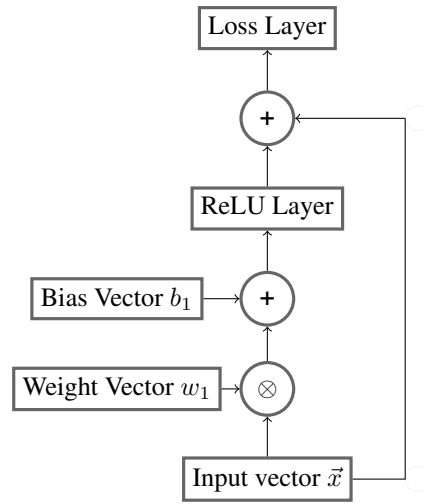

We encode the \textit{loss layer} to capture the \textit{distance-to-satisfaction} $\psi$ for a given candidate hyponym-hypernym pair ($x,y$), defined as follows:
\begin{equation}\label{per_instance_loss}
\psi (x,y) = \sum_{i=1}^d max(0, \epsilon + f(\vec{x})_i - \vec{y}_i)
\end{equation}
where the sum is taken over all the components of the participating dimensions, and $\epsilon$ is a scalar hyper-parameter.

The network is trained by feeding positive and negative examples derived from a training set $\mathcal{T}$ containing \emph{is-a} relations and their corresponding scores. Each positive training instance consists of a pair $(x, \mathcal{H}_x)$, where $\mathcal{H}_x$ is the set of candidate hypernyms of $x$ in the training data. Negative instances for a given term $x$, denoted by $\mathcal{H}^{\prime}_x$, are generated by selecting terms uniformly at random from $\mathcal{E}$.
More formally, for a given candidate hyponym term $x$, let 
\begin{equation} \label{cand-hyp-terms}
\mathcal{H}_x = \{e \in \mathcal{E} \;| \;(x, e, s) \in \mathcal{T}\}
\end{equation}
denote all the candidate hypernym terms of $x$, and let 
\begin{equation} \label{neg-hyp-samples}
\mathcal{H}^{\prime}_x = \{e \in \mathcal{E} | (x, e, .) \notin \mathcal{T}\}
\end{equation}
denote negative hypernym samples for $x$. 
Negative hypernym terms are sampled at random from $\mathcal{E}$, and as many negative samples are generated that satisfy $|\mathcal{H}_x| + |\mathcal{H}^{\prime}_x| = k$, a constant (hyper-parameter for the model).

The probability of ($x, y$) being a true hyponym-hypernym pair is then calculated using an approach analogous to \textit{Boltzmann distribution} as follows,
\begin{equation}\label{probability}
p(x,y) = \frac{e^{-\psi (x,y)}}{\sum_{z \in \mathcal{H}_x \cup \mathcal{H}^{\prime}_x} e^{-\psi (x, z)}}
\end{equation}

Equation \ref{probability} is used for training, while during scoring, the probability that a pair $(x,y)$ exhibits \textit{hypernymy} relationship is given by, 
\begin{equation}
p(x,y) = \frac{e^{-\psi (x,y)}}{\sum_{z \in \mathcal{H}} e^{-\psi (x,z)}}
\end{equation}\label{scoring}

whereas, the most likely hypernym term $y^*$ for a given hyponym term $x$ is given by,
\begin{equation}
    y^* = \argmax_{y \in \mathcal{H}} \frac{e^{-\psi (x,y)}}{\sum_{z \in \mathcal{H}} e^{-\psi (x,z)}}
\end{equation}

Here, $\mathcal{H}$ denotes the list of all hypernym terms observed in the training set $\mathcal{T}$.

Finally, we define the loss function $\mathcal{J}$ using a weighted negative log-likelihood criterion (w-NLL) defined as follows, 
\begin{equation}
\mathcal{J} = - \sum_{(x,y,s) \in \mathcal{T}} s \log p(x,y)
\end{equation}
where $s$ represents the relative importance of the loss associated with pair $(x,y) \in \mathcal{T}$. 

\section{Theoretical Analysis}
\label{theoreticalanalysis}

\textit{Hypernymy} relations within a taxonomy satisfy two properties: \textit{asymmetry} and \textit{transitivity}. The \textit{asymmetry} property states that given two distinct terms $x, y \in \mathcal{E}$, if $x \prec y$ then $y \nprec x$. 
The \textit{transitive} property states that given three distinct terms $x,y,z \in \mathcal{E}$, if $x$ $\prec$ $y$ and $y$ $\prec$ $z$ then $x$ $\prec$ $z$. 

In this section we analytically demonstrate that the neural network architecture depicted in Fig. \ref{architecture}, whose forward pass expressions are given by equations \ref{reludesc} and \ref{resconn}, satisfy \textit{asymmetry} and \textit{transitive} properties. 

As described by equation \ref{per_instance_loss}, our proposed model assigns a zero loss for a given hyponym-hypernym pair $(x,y)$ if the learned model satisfies $f(\vec{x}) < \vec{y}$ element-wise. This formulation of the \emph{loss layer} puts forth the following constraint that defines our model, 

\begin{equation} \label{modelineq}
x \prec y \iff f(\vec{x})_i <  \vec{y}_i,  \; \forall i
\end{equation}

In other words, the relation $\prec$ is satisfied if and only if $f(\vec{x}) < \vec{y}$, component-wise. 
In the rest of this section, we show that under the assumption of [\ref{modelineq}], our proposed model for \emph{hypernymy} relation satisfies \emph{asymmetry} and \emph{transitivity}.

\begin{theorem}
Expression \ref{modelineq} satisfies asymmetry.
\end{theorem}
\begin{proof}
Let $x \prec y$. Then, it follows expression \ref{modelineq} that $f(\vec{x}) < \vec{y}$ component wise. We need to show that $y \nprec x$. Using the definition of equation \ref{modelineq}, it is enough to show $f(\vec{y}) \geq \vec{x}$ component wise. 

Now, using equation \ref{resconn}, we have $f(\vec{y}) = \vec{y}+g(\vec{y})$. From the definition of function $g$, it is clear that $g(\vec{x}) \geq 0$ component wise. Thus, applying this inequality to the previous expression, we have $f(\vec{y}) \geq \vec{y}$ component wise. On similar lines, we can also show that 
\begin{equation} \label{asymmetry_proof}
f(\vec{x}) \geq \vec{x} \;\; \forall x \in \mathcal{E}
\end{equation}
component wise.

We now have $f(\vec{y}) \geq \vec{y} > f(\vec{x}) \geq \vec{x}$ component wise. The middle inequality holds, since we assume $x \prec y$; in other words, $f(\vec{x}) < \vec{y}$ holds component wise. Thus expression \ref{modelineq} satisfies asymmetry. 

\end{proof}

\begin{theorem}
Expression \ref{modelineq} satisfies transitivity.
\end{theorem}

\begin{proof}
Let $x \prec y$ and $y \prec z$. Then, it follows from expression \ref{modelineq} that $f(\vec{x}) < \vec{y}$ and $f(\vec{y}) < \vec{z}$, component wise. We need to show that $x \prec z$ or, alternatively, that $f(\vec{x}) < \vec{z}$. 

Generalizing equation \ref{asymmetry_proof}, we have,
$\forall e \in \mathcal{E}, f(\vec{e}) \geq \vec{e}$ component wise. Using this observation, we have $f(\vec{x}) < \vec{y} \leq f(\vec{y})< \vec{z}$ component wise. Note that the middle inequality holds from the aforementioned observation. This proves that Expression \ref{modelineq} satisfies transitivity. 
\end{proof}

\section{Generalizing SPON to OOV} \label{oovterms}

The proposed SPON model is able to learn embedding for terms appearing in the training data (extracted either using Hearst-like patterns or provided via a manually labelled training set). However, for tasks wherein one needs to induce \textit{hypernymy} relationships automatically from a text corpus, Hearst-like patterns usually are not exhaustive.

Yet, there is often a practical requirement in most applications to assign 
OOV to their most likely correct type(s). Designing a system that fulfills this requirement is highly significant since it allows the creation of \textit{hypernymy} relationships from a given text corpus, avoiding the problem of sparsity that often characterizes most knowledge bases. The basic idea is to use an \textit{augmented} SPON approach that leverages distributional similarity metrics between words in the same corpus. This is formally described as follows.

For a given domain, let $\mathcal{I}^{trial}$ and $\mathcal{O}^{trial}$ denote the in-vocabulary and OOV input trial hyponym terms; and let $\mathcal{I}^{test}$ and $\mathcal{O}^{test}$ denote the in-vocabulary and OOV input test hyponym terms respectively. Let $\mathcal{E}^{train}$ denote all the terms observed in the list of training \textit{hyponym-hypernym} pairs, and let $\mathcal{H}^{train}$ denote the list of known hypernyms obtained from the list of training pairs. The hyponym terms from $\mathcal{I}^{trial}$ and $\mathcal{I}^{test}$ are handled by our proposed SPON model, i.e. top-ranked hypernyms for each hyponym term are generated via our model.

The rest of this section deals with how to generate top ranked hypernyms for each hyponym term within $\mathcal{O}^{trial}$ and $\mathcal{O}^{test}$ respectively. Let $Y_x$ be the random variable denoting the hypernym assignment for an OOV term $x \in \mathcal{O}^{test}$ (Similar approach holds for OOV terms from $\mathcal{O}^{trial}$). The probability of the random variable $Y_x$ taking on the value $c \in \mathcal{H}^{train}$ is then given by,
\begin{multline*}
    P(Y_x = c | x) = \sum_{h \in \mathcal{E}^{train}} P(Y_x=c, h | x) \\ 
                   = \sum_{h \in \mathcal{E}^{train}} P(Y_x=c|h,x).P(h|x) \numberthis \label{oov_eq1}
\end{multline*}    

The first equality in the above expression is a direct consequence of \textit{Marginalisation} property in probability, whereas the second equality merely represents the \textit{marginal} probability in terms of \textit{conditional} probability. \newline
We now make a \textit{conditional independence} assumption, i.e. $Y_x \ci x \;|\; h $, or in other words, ignoring the subscript for brevity we have,  $P(Y_x = c |h,x) = P(Y=c|h)$. Using this assumption, we can rewrite Equation \ref{oov_eq1} as,
\begin{align*}
    P(Y_x = c | x) &= \sum_{h \in \mathcal{E}^{train}} P(Y=c|h).P(h|x) \\ 
                   &\approx \sum_{h \in S^{p}_x} P(Y=c|h).f(h|x) \numberthis \label{oov_eq2}
\end{align*}
where $f$ is a scoring function that provides a score between $[0,1]$, and $S^{p}_x$ contains \textit{p-terms} from $\mathcal{E}^{train}$ that provide top-k largest values for the scoring function $f$. In practice, we first normalize the values of $f(h|x)$ where $ h \in S^{p}_x$ using a \textit{softmax} operation, before computing the weighted sum, as per Equation \ref{oov_eq2}. Also, note that $p$ is a hyper-parameter in this model.

Looking back at Equation \ref{oov_eq2}, we notice that the first part of the summation, i.e. $P(Y=c|h)$, can be obtained directly from our proposed SPON model, since $h \in \mathcal{E}^{train}$. In addition, we model the function $f(h|x)$ as cosine-similarity between the vectors for the term $h$ and $x$, wherein the vectors are trained via a standard Word2Vec model pre-built on the corresponding tokenized corpus for the given benchmark dataset.

Summarizing, given a query OOV term within the \emph{trial} or \emph{test} fold of any dataset, our proposed model follows the aforementioned strategy to generate a list of \emph{hypernym} terms that have been ranked using the formula in Equation \ref{oov_eq2}.

It should be \emph{clearly} pointed out that our aforementioned proposed OOV strategy is not a stand-alone strategy, rather its performance is inherently \emph{dependent} of SPON.


\section{Unsupervised Benchmarks and Evaluation}
\label{unsupervised}

SPON is intrinsically supervised because it requires example \emph{is-a} pairs for training. However, it can also be applied to unsupervised hypernymy task, provided that example \emph{is-a} pairs are generated by an external unsupervised process such as Hearst-like patterns.

\subsection{Benchmarks}
In the unsupervised setting, no gold training data is provided and the system is supposed to assess the validity of test data, provided as a set of pairs of words. A small validation dataset is also provided which is used for tuning hyper-parameters.

We evaluated our approach on two tasks. The first one is \emph{hypernym detection} where the goal is to classify whether a given pair of terms are in a hypernymy relation. The second task is \emph{direction prediction}, i.e. to identify which term in a given pair is the hypernym. We use the \textit{same datasets, same settings, same evaluation script and same evaluation metrics} as \citet{roller2018hearst}. Table \ref{UnsupervisedDataSets} shows the dataset statistics for unsupervised benchmarks, wherein the split into validation/test folds is already given.\footnote{The only exception to this is BIBLESS dataset comprising of 1669 pairs, for which the split is not provided a priori.} 

For detection, Average Precision is reported on 5 datasets, namely BLESS \cite{baroni2011we}, LEDS \cite{baroni2012entailment}, EVAL \cite{santus2015evalution}, WBLESS \cite{weeds2014learning} and SHWARTZ \cite{shwartz2016improving}. While for direction, Average Accuracy is reported on 3 datasets, which are BIBLESS \cite{kiela2015exploiting}, BLESS and WBLESS. We refer the readers to \citet{roller2018hearst} for details about these datasets.

\begin{table}
    \centering 
    \begin{tabular}{c|c|c}
    \hline
    Dataset & Valid & Test  \\ \hline \hline
    BLESS  & 1,453 & 13,089 \\ 
    EVAL & 736 & 12,714 \\ 
    LEDS & 275 & 2,495 \\ 
    SHWARTZ & 5,236 & 47,321 \\ 
    WBLESS & 167 & 1,501 \\ 
    \hline  
    \end{tabular}
    \caption{Statistics for benchmark datasets used in unsupervised \emph{hypernym detection} and \emph{direction prediction} tasks. The columns represent the number of \emph{hyponym-hypernym} pairs within the \emph{validation} and \emph{test} folds respectively.}
    \label{UnsupervisedDataSets}
\end{table}

\begin{table}
    \centering
    \begin{tabular}{|c|c|p{2cm}|p{2cm}|}
      \hline
        & \small English & \small Italian/Spanish & \small Music/Medical \\ \hline \hline
        Train & 1500 & \centering 1000 & 500 \\ \hline
        Trial & 50 & \centering 25 & 15  \\ 
        Test & 1500 & \centering 1000 & 500 \\
      \hline
    \end{tabular} 
    \caption{Number of hyponyms in different datasets within SemEval 2018 \emph{hypernym discovery} task.}
\label{tab:SemEval18_dataset}
\end{table}

\begin{figure}[ht]
   \centering
   \includegraphics[width=0.9\columnwidth]{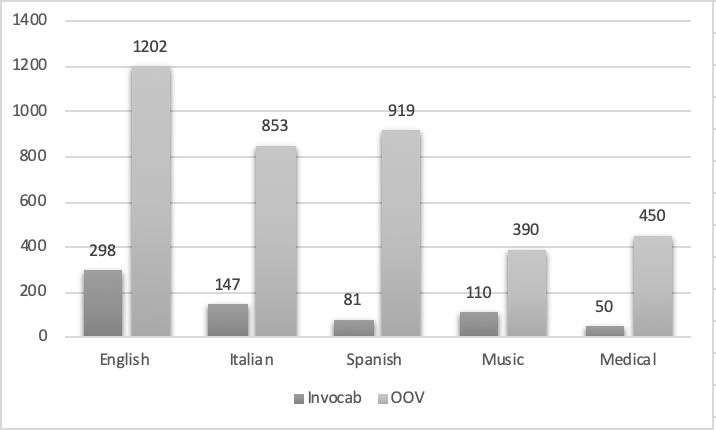}
   \caption{Breakdown of hyponym terms within the test fold for each dataset in the hypernym discovery task. By \emph{Invocab} we mean hyponym terms within \textit{test} fold that have been observed while training SPON, whereas by OOV we mean new hyponym terms that have been exclusively observed for the first time in \textit{test} fold and not seen during \textit{training}.}
   \label{oov}
 \end{figure}

\begin{table*}
    \centering
    \begin{tabular}{p{3cm}|c|c|c|c|c||c|c|c}
      \hline
      & \multicolumn{5}{c|}{Detection (Average Precision)} &  \multicolumn{3}{|c}{Direction (Average Accuracy)} \\ \hline
      & \small BLESS & \small EVAL & \small LEDS & \small SHWARTZ & \small WBLESS & \small BLESS & \small WBLESS & \small BIBLESS  \\ \hline \hline
      Count based p(x,y) & .49 & .38 & .71 & .29 & .74 & .46 & .69 & .62 \\
      ppmi(x,y) & .45 & .36 & .70 & .28 & .72 & .46 & .68 & .61 \\
      SVD ppmi(x,y) & .76 & .48 & .84 & .44 & .96 & .96 & .87 & .85 \\
      HyperbolicCones & \bf .81 & \bf .50 & .89 & \bf .50 & \bf .98 & .94 & .90 & .\bf 87 \\ \hline
      Proposed SPON & \bf .81 & \bf .50 & \bf .91 & \bf .50 & \bf .98 & \bf .97 & \bf .91 & \bf .87 \\
      \hline
    \end{tabular} 
    \caption{Results on the unsupervised \emph{hypernym detection} and \emph{direction prediction} tasks. The first three rows of results are from \citet{roller2018hearst}. The \textit{HyperbolicCones} results were reported by \citet{le2019Hyperbolic}. The improvements in LEDS and BLESS benchmark are statistically significant with \emph{two-tailed p values} being 0.019 and $\leq$ 0.001 respectively.}
    \label{ResultsTable}
\end{table*}

\begin{table}[ht]
    \centering
    \begin{tabular}{c|c|c|c|c}
    \hline 
    & \small BLESS & \small EVAL & \small LEDS & \small WBLESS \\ \hline \hline 
    \small \emph{RELU}+\emph{Residual} & \bf .81 & \bf .50 & \bf .91 & \bf .98  \\
    \small \emph{RELU} Only & .73 & .49 & .82 & .96 \\ 
    \small \emph{Tanh}+\emph{Residual} & .79 & .49 & .90 & \bf .98 \\ 
    \hline 
    \end{tabular}
    \caption{Ablation tests reporting \emph{Average Precision} values on the unsupervised \emph{hypernym detection} task, signifying the choice of \emph{layers} utilized in our proposed SPON model. The first row represents SPON i.e.\ a \emph{RELU} layer followed by a \emph{Residual} connection. The second row removes the \emph{Residual} connection, whereas the third row substitutes the \emph{non-negative} activation layer \emph{RELU} with \emph{Tanh} that can take negative values.}
    \label{tab:Ablation1}
\end{table}

\begin{table}[ht]
    \centering
    \begin{tabular}{c|c}
    \hline 
    Method & Average Precision \\ \hline 
    OE \cite{Vendrov:2016} & 0.761 \\ 
    Smoothed Box \cite{li2018smoothing} & 0.795 \\ 
    SPON (Our Approach) & \bf 0.811 \\ 
    \hline 
    \end{tabular}
    \caption{Results on the unsupervised \emph{hypernym detection} task for BLESS dataset. With 13,089 test instances, the improvement in Average Precision values obtained by SPON as compared against Smoothed Box model is statistically significant with \emph{two-tailed p value} equals $0.00116$.}
    \label{tab:compareWithICLRPubs}
\end{table}

\begin{table*}
    \centering
    \begin{tabular}{c|c|c|c||c|c|c||c|c|c}
      \hline
      & \multicolumn{3}{c|}{\bf English} & \multicolumn{3}{|c|}{\bf Spanish} & \multicolumn{3}{|c}{\bf Italian} \\ \hline
      & \small MAP & \small MRR & \small P@5 & \small MAP & \small MRR & \small P@5 & \small MAP & \small MRR & \small P@5 \\ \hline \hline
      CRIM & 19.78 & 36.10 & 19.03 & -- & -- & -- & -- & -- & -- \\ \hline
      NLP\_HZ & 9.37 & 17.29 & 9.19 & 20.04 & 28.27 & 20.39 & 11.37 & 19.19 & 11.23 \\ \hline
      300-sparsans & 8.95 & 19.44 & 8.63 & 17.94 & 37.56 & 17.06 & 12.08 & 25.14 & 11.73 \\ \hline \hline
      SPON & \bf 20.20 & \bf 36.95 & \bf 19.40 & \textbf{32.64} & \textbf{50.48} & \textbf{32.76}& \textbf{17.88} &\textbf{29.80} & \textbf{17.95} \\ \hline
    \end{tabular} 
    \caption{Results on SemEval 2018 General-purpose hypernym discovery task. \texttt{CRIM}, \texttt{NLP\_HZ}, and \texttt{300-sparsans} are the corresponding best systems on English, Spanish and Italian datasets (see Section \ref{relatedwork}).}
    \label{tab:results_general}
\end{table*}

\begin{table}[t]
    \centering 
    \begin{tabular}{|c|c|c|c|}
      \hline
      & \multicolumn{3}{c|}{\bf Music} \\ \hline
      & \small MAP & \small MRR & \small P@5 \\ \hline
      CRIM &  40.97 & 60.93 & 41.31 \\ \hline
      SPON & \textbf{54.70} & \textbf{71.20} & \textbf{56.30} \\  \hline
      & \multicolumn{3}{|c|}{\bf Medical} \\ \hline
      & \small MAP & \small MRR & \small P@5 \\ \hline
      CRIM & \textbf{34.05} & \textbf{54.64} & \textbf{36.77} \\ \hline
      SPON & 33.50 & 50.60 & 35.10 \\ \hline
    \end{tabular} 
    \caption{Results on SemEval 2018 Domain-specific hypernym discovery task. \texttt{CRIM} is the best system on the domain specific datasets.}
    \label{tab:results_domain}
\end{table}

We adopted the approach of \citet{roller2018hearst} where a list $\mathcal{L}$ of hyponym-hypernym pairs $(x, y)$ is extracted using a Hearst-like pattern-based system. This system consists of 20 Hearst-like patterns applied to a concatenation of Wikipedia and Gigaword corpora, to generate $\mathcal{L}$.

Each pair within $\mathcal{L}$ is associated with a count $c$ (how often $(x, y)$ has been extracted). Positive Mutual Information (PPMI) \cite{bullinaria2007extracting} for each \emph{(x, y)} is then calculated. Let the size of $\mathcal{E}$ be $m$ and let $M \in  \mathbb{R}^{m \times m}$ be the PPMI matrix. We use a similar scoring strategy as \citet{roller2018hearst}, i.e. truncated SVD approach to generate term embeddings, and score each pair using cosine similarity. This creates a modified list $\mathcal{T}$, which is the input for SPON (as mentioned in Section \ref{SPON}).

In order to be directly comparable to \citet{roller2018hearst} and \citet{le2019Hyperbolic}, we used the \emph{same input file} of \citet{roller2018hearst} containing candidate hyponym-hypernym-count triples, i.e.\ a total of 431,684 triples extracted using Hearst-like patterns from a combination of Wikipedia and Gigaword corpora. 

We used the following hyper-parameter configuration for the \emph{rank} parameter of the SVD based models: 50 for BLESS, WBLESS, and BIBLESS; 25 for EVAL, 100 for LEDS and 5 for SHWARTZ. Optimal hyper-parameter configurations for our proposed SPON model were determined empirically using validation fold for the benchmark datasets.

For each experiment, the embedding dimensions \emph{d} were chosen out of \{100, 200, 300, 512, 1024\}, whereas the $\epsilon$ parameter was chosen out of $\{10^{-1}, 10^{-2}, 10^{-3}, 10^{-4}\}$. 
$k$ is set to 1000 for all experiments. For example, in Table \ref{ResultsTable}, the SPON model used the following hyper-parameters on BLESS dataset: $d=300, \epsilon=0.01$.

In addition, we used $L_1$ regularization for model weights, and also used \emph{dropout} with \emph{probability} of $0.5$. Adam optimizer was used with default settings. In addition, the term vectors in our model were initialized uniformly at random, and are constrained to have unit $L_2$ norm during the entire training procedure. Furthermore, an early stopping criterion of 20 \emph{epochs} was used.

\subsection{Evaluation}
We use the same evaluation script as provided by \citet{roller2018hearst} for evaluating our proposed model. Table \ref{ResultsTable} shows the results on the unsupervised tasks of \emph{hypernym detection} and \emph{direction predictions}, reporting average precision and average accuracy, respectively. 

The first row titled \emph{Count based} (in Table \ref{ResultsTable}) depicts the performance of a  Hearst-like Pattern system baseline, that uses a \emph{frequency} based threshold to classify candidate hyponym-hypernym pairs as positive (i.e.\ exhibiting \emph{hypernymy}) or negative (i.e.\ not exhibiting \emph{hypernymy}). The \emph{ppmi} approach in Table \ref{ResultsTable} builds upon the \emph{Count based} approach by using Pointwise Mutual Information values for classification. \emph{SVD ppmi} approach, the main contribution from \citet{roller2018hearst} builds low-rank embeddings of the PPMI matrix, which allows to make predictions for unseen pairs as well. 

\emph{HyperbolicCones} is the SOTA \cite{le2019Hyperbolic} in both these tasks. The final row reports the application of SPON (on the input provided by SVD ppmi) which is an original contribution of our work. Results clearly show that SPON achieves SOTA results on all datasets. In fact, on three datasets, SPON outperforms \emph{HyperbolicCones}. Furthermore, improvements in LEDS and BLESS benchmarks are statistically significant with \emph{two-tailed p values} being 0.019 and $\leq$ 0.001 respectively.

A plausible explanation for this improved performance might be due to the fact that \emph{hypernymy} relationships are better represented as \emph{Directed Acyclic Graphs} (DAGs) rather than trees \cite{suchanek2008yago}, and we believe that SPON is more suitable to represent \emph{hypernymy} relationships as opposed to \emph{HyperbolicCones} in which the constant negative curvature strongly biases the model towards trees. 

\paragraph{Ablation Tests.} The \emph{analysis} in Section \ref{theoreticalanalysis} which shows that our choice of function $f$ satisfies \emph{asymmetry} and \emph{transitive} properties, holds true because $f$ satisfies $f(\vec{x}) \geq \vec{x}$ component-wise. We have chosen to define $f$ as a non-negative activation function \emph{RELU} followed by a \emph{Residual} layer. In this section, we perform \emph{two} sets of ablation experiments, \emph{first} where we remove the \emph{Residual} connections altogether, and \emph{second} where we replace the non-negative activation function \emph{RELU} with \emph{Tanh} that can take on negative values. 

Table \ref{tab:Ablation1} shows the results for each of these ablation experiments, when evaluated on the unsupervised \emph{hypernym detection} task across \emph{four} datasets chosen randomly. Removing the \emph{Residual} layer and using \emph{RELU} activation function only, violates the aforementioned component-wise inequality $f(\vec{x}) \geq \vec{x}$, and has the worst results out of the three. On the other hand, using \emph{Residual} connections with \emph{Tanh} activations may not violate the aforementioned inequality, since, it depends upon the \emph{sign} of the activation outputs. This \emph{argument} is supported by the results in Table \ref{tab:Ablation1}, wherein using \emph{Tanh} activations instead of \emph{RELU} almost provides \emph{identical} results, except for the BLESS dataset. 
Nevertheless, the results in Table \ref{tab:Ablation1} show that encouraging \emph{asymmetry} and \emph{transitive} properties for this \emph{task}, in fact improves the results as opposed to not doing the same.

Furthermore, Table \ref{tab:compareWithICLRPubs} illustrates the results on the unsupervised \emph{hypernym detection} task for BLESS dataset, wherein we compare our proposed SPON model to other supervised SOTA approaches for hypernym prediction task, namely Order Embeddings (OE) approach as introduced by \cite{Vendrov:2016}, and Smoothed Box model as introduced by \cite{li2018smoothing}. We run the OE and Smoothed Box experiments using the codes provided with those papers. 

In addition, we used the validation fold within BLESS dataset to empirically determine optimal hyper-parameter configurations, and settled on the following values: For OE, we used an embedding dimensions of 20, margin parameter of 5, generated \emph{one} negative example for every positive instance using so-called \emph{contrastive} approach. For Smoothed Box model, we used an embedding dimensions of 50 and generated \emph{five} negatives per training instance. In either case, we observed that using the entire set of \emph{is-a} pairs extracted by the Hearst-like patterns (without employing a frequency based cutoff) for training provided the best performance.    

From Table \ref{tab:compareWithICLRPubs}, it is clear that SPON performs much better (by atleast 1.6\%) as compared to Smoothed Box model as well as Order Embedding model in an Unsupervised benchmark dataset.

\begin{table}
    \centering
    \resizebox{.95\columnwidth}{!}{
    \begin{tabular}{c|c}
      \hline
        \bf {Term} & \textbf{Predicted hypernyms} \\ \hline \hline
        dicoumarol & \small \underline{drug}, carbohydrate, acid, person, ... \\ 
        Planck & \small \textbf{person}, \underline{particle}, physics, \underline{elementary particle}, ...  \\ 
        Belt Line & \small main road, \textbf{infrastructure}, expressway, ... \\
        relief & \small service, assistance, resource, ... \\ \hline
        honesty & \small \underline{virtue}, ideal, moral philosophy, ... \\
        shoe & \small \textbf{footwear}, shoe, \textbf{footgear}, overshoe, ... \\ 
        ethanol & \small \textbf{alcohol}, \underline{fuel}, person, \textbf{fluid}, resource, ... \\ 
        ruby & \small \underline{language}, \underline{precious stone}, \underline{person}, ... \\ 
     
    \end{tabular}}
    \caption{Examples of ranked predictions (from left-to-right) made by our system on a set of eight \emph{randomly} selected test queries from SemEval 2018 English dataset. The top four query terms are OOV, while the bottom ones are In-vocabulary. Hypernyms predicted by SPON that matches the gold annotations are highlighted in \textbf{bold}, while we use \underline{underline} for predictions that we judge to be correct but are missing in the gold standard expected hypernyms.}
    
\label{tab:SemEval18_analysis}
\end{table}

\section{Supervised Benchmarks and Evaluation}
\label{supervised}

In the supervised setting, a system has access to a large corpus of text from where training, trial, and test \emph{is-a} pairs are extracted, and labeled manually.

\subsection{Benchmarks}

We used the benchmark of \textit{SemEval 2018 Task on Hypernym Discovery}. The task is defined as \textit{``given an input term, retrieve its hypernyms from a target corpus"}. For each input hyponym in the test data, a ranked list of candidate hypernyms is expected. The benchmark consists of five different subtasks covering both general-purpose (multiple languages -- English, Italian, and Spanish) and domain-specific (Music and Medicine domains) hypernym discovery . 

\subsection{Experimental Settings}\label{exp_settings}

This subsection describes the technical solution we implemented for the SemEval tasks, more specifically, the strategies for training dataset augmentation, handling OOV terms, and the hyperparameter optimization.

For the corpora in English, we augmented the training data using automatically extracted pairs by an unsupervised Hearst-like Pattern (HP) based system, following an approach similar to that described by \citet{bernier2018crim}, the best system in English, Medical and Music hypernymy subtasks in SemEval 2018.  Henceforth, we refer to our pattern-based approach as the \texttt{HP} system. 

The HP system uses a fixed set of Hearst-like patterns (e.g. \emph{``y such as x''}, \emph{``y including than $x_1$, $x_2$''}, etc) to extract pairs from the input corpus. Then, it filters out any of these pairs where either \emph{x} (hyponym) or \emph{y} (hypernym) is not seen in the corresponding vocabulary provided by the organizers of the shared task. It also discards any pair where \emph{y} is not seen in the corresponding gold training data. 

Following that, the HP system makes a directed graph by considering each pair as an edge and the corresponding terms inside pair as nodes. The weight of each edge is the count/frequency of how often \emph{(x, y)} has been extracted by the Hearst-like patterns.

It also excludes any cycle inside the graph; e.g. if \emph{(x, y)}, \emph{(y, z)} and \emph{(z, x)} then all these edges (i.e. pairs) were discarded. Finally, it discards any edge that has a value lower than a frequency threshold, \emph{ft}. We set \emph{ft}=10 for \emph{English}. For \emph{Medical} and \emph{Music}, we set \emph{ft}=2.

The \emph{is-a} pairs obtained from HP system are then merged with the corresponding training gold pairs (i.e. treated them equally) to form a larger training set for English, Medical and Music datasets. As a result of this step, the number of total unique training pairs for English, Medical and Music increased to 17,903 (from 11,779 training gold pairs), 4,593 (from 3,256) and 6,282 (from 5,455) correspondingly. 

The dataset statistics for the \textit{general-purpose} and \textit{domain-specific} hypernym discovery tasks are mentioned in Table \ref{tab:SemEval18_dataset}. It is evident that a significant fraction of the terms in the Trial/Test fold is OOV (see Figure~\ref{oov}), therefore SPON is not able to make any assessment about them. Therefore, in order to handle OOV cases, we represent all terms in the dictionary provided by the SemEval organizers via Word2Vec vectors acquired from the given text corpus.

The \textit{dimensions} $d$ for SPON model was chosen from $\{50, 100, 200\}$, whereas the parameter $p$ (for handling OOV terms) was chosen from $\{2,3,5,8,10\}$. Parameter $k$ (from Equation \ref{neg-hyp-samples}) was chosen from $\{100,200,500\}$. The \textit{regularization}, \textit{dropout} and \textit{initialization} strategies are exactly similar to Section \ref{exp_settings}. An early stopping criterion of \textit{50 epochs} was used. 

\subsection{Evaluation}

We use the \emph{scorer} script provided as part of \emph{SemEval-2018} Task 9 for evaluating our proposed model. Table~\ref{tab:results_general} shows the results on the three general purpose domains of English, Spanish, and Italian respectively. For brevity, we compare only with the SOTA, \textit{i.e.}, the best system in each task. Performances of all the systems that participated in \textit{SemEval 2018 Task} on Hypernym Discovery can be found in \cite{camacho-collados-etal-2018-semeval}.  Similarly, Table~\ref{tab:results_domain} shows the results on the two domain-specific tasks of music and medical domain corpora. SPON outperforms the SOTA systems in all tasks except for the medical domain in which it achieves comparable results. It is worthwhile to notice that SPON is fully domain-agnostic, \textit{i.e.}, it neither uses any domain-specific approaches, nor any domain-specific external knowledge. 
We provide an illustration of the output of our system in Table \ref{tab:SemEval18_analysis}, showing a sample of randomly selected terms and their corresponding ranked predictions.

\section{Conclusion and Future Work}
\label{conclusion}

In this paper, we introduced SPON, a novel neural network architecture that models \textit{hypernymy} as a strict partial order relation. We presented a materialization of SPON, along with an \textit{augmented} variant that assigns \textit{types} to OOV \textit{terms}. An extensive evaluation over several widely-known academic benchmarks demonstrates that SPON largely improves (or attains) SOTA values across different tasks.

There are so many benchmark datasets for hypernymy prediction task with different evaluation settings (supervised and unsupervised). None of the recent approaches choose to report results on all of them which makes it difficult to decide whether any one of them performs consistently well over others. Our paper fills this void.

In the future, we plan to explore how to extend SPON in two directions. On the one hand, we plan to analyze how to use SPON for the taxonomy construction task (\textit{i.e.}, constructing a hierarchy of hypernyms instead of flat \emph{is-a} pairs). On the other hand, we plan to generalize our work to relations other than \emph{is-a}.

\small   
\bibliography{AAAI-DashS.9092}
\bibliographystyle{aaai}

\end{document}